\documentclass[wcp]{jmlr}

\usepackage{longtable}

\usepackage{booktabs}

\usepackage{dsserif}
\usepackage[english]{babel}
\usepackage{lscape}

\makeatletter
\let\Ginclude@graphics\@org@Ginclude@graphics 
\makeatother

\jmlrvolume{189}
\jmlryear{2022}
\jmlrworkshop{ACML 2022}

\title[KMGM]{Kernelized multi-graph matching}

  \author{\Name{Fran{\c c}ois-Xavier Dup\'e} \Email{francois-xavier.dupe@univ-amu.fr}\\
  \addr Aix Marseille Univ, CNRS, LIS, Marseille, France
  \AND
  \Name{Rohit Yadav} \Email{rohit.yadav@univ-amu.fr}\\
  \addr Aix Marseille Univ, CNRS, INT, Inst Neurosci Timone, Marseille, France \\
  \addr Aix Marseille Univ, CNRS, LIS, Marseille, France\\
  \addr Aix Marseille Univ, Institut Marseille Imaging, Marseille, France
  \AND
  \Name{Guillaume Auzias} \Email{guillaume.auzias@univ-amu.fr} \\
  \Name{Sylvain Takerkart} \Email{sylvain.takerkart@univ-amu.fr}\\
  \addr Aix Marseille Univ, CNRS, INT, Inst Neurosci Timone, Marseille, France
 }

 \editors{Emtiyaz Khan and Mehmet G\"{o}nen}

 \DeclareMathOperator*{\argmin}{argmin}

\begin{document}

\maketitle

\begin{abstract}
  Multigraph matching is a recent variant of the graph matching problem. In this
  framework, the optimization procedure considers several graphs and enforces the
  consistency of the matches along the graphs. This constraint can be formalized as a
  cycle consistency across the pairwise permutation matrices, which implies the definition
  of a universe of vertex~\citep{pachauri2013solving}. The label of each vertex is encoded by
  a sparse vector and the dimension of this space corresponds to the rank of the bulk
  permutation matrix, the matrix built from the aggregation of all the pairwise
  permutation matrices. The matching problem can then be formulated as a non-convex
  quadratic optimization problem (QAP) under constraints imposed on the rank and the
  permutations.  In this paper, we introduce a novel kernelized multigraph matching
  technique that handles vectors of attributes on both the vertices and edges of the graphs,
  while maintaining a low memory usage.  We solve the QAP problem using a projected power
  optimization approach and propose several projectors leading to improved stability of
  the results. We provide several experiments showing that our method is competitive
  against other unsupervised methods.
\end{abstract}
\begin{keywords}
graph matching, kernel, multi-graph matching, non-convex optimization
\end{keywords}

\section{Introduction \& Related work}

Graph matching is general problem with many applications such as e.g. object recognition
and registration, or shape matching to cite a few.
Let $G_1$ and $G_2$ be two graphs with respectively $n_1$ and $n_2$ vertices. The pairwise
matching problem is often cast as a quadratic assignment problem (QAP), using the Lawler's
formula~\citep{lawler1963quadratic},
\begin{equation}
  \label{eq:lawler}
  \max_{X\in\mathcal{P}_{n_1,n_2}} {\mathrm{vec}(X)}^{T} K^{e} \mathrm{vec}(X) + \mathrm{tr}(K^{v} X)~,
\end{equation}
with $\mathcal{P}_{n_1,n_2}$ the set of permutation matrix of size $n_1 \times n_2$,
$K^{v}\in\mathbb{R}^{n_1\times n_2}$ the vertex affinity matrix and
$K^{e} \in \mathbb{R}^{n_1 n_2 \times n_1 n_2}$ the edge affinity matrix. These affinity matrices can be built using kernels, then $K^{v}$ and $K^{e}$ are
both Gram matrix build using respectively a vertex kernel and an edge kernel~\citep{zhang2019kergm}.


As $K^{e}$ could be a very large matrix, usually the Koopmans-Beckmann's
QAP~\citep{koopmans1957assignment} is preferred,
\begin{equation}
  \label{eq:koopmans}
  \max_{X\in\mathcal{P}_{n_1,n_2}} \mathrm{tr}(XA_1XA_2) + \mathrm{tr}(K^{v}X)~,
\end{equation}
with $A_1$ and $A_2$ respectively the adjacency matrices of $G_1$ and $G_2$. 
\eqref{eq:koopmans} is a special case of Lawler's QAP with $K^{e} = A_1 \otimes A_2$ (the
Kronecker product between the two adjacency matrices).%

Recently,~\cite{zhang2019kergm} showed that by using specific kernels for computing the
edge affinity matrix,~\eqref{eq:lawler} has a more memory manageable writing. Furthermore,
their formulation with kernels allows to match a large diversity of graphs, as kernels can
easily manage labels or vectors of attributes. They also propose a regularization scheme combined with a
Frank-Wolfe optimization leading to very good performances compared to state of the art
methods and good robustness to noise both on attributes and structure.

Multigraph matching is an extension of the graph matching problem, where one aims at
matching several graphs at once while enforcing coherence between the vertex assignments. In
order to enforce the coherence, the idea is to search for a cycle
consistency~\citep{pachauri2013solving} between the permutation matrices, i.e. if $X_{i,j}$
is the permutation matrix between graphs $i$ and $j$, for a graph $k$ we have
$X_{i, j} \approx X_{i,k} X_{k,j}$. It is an approximation since the graphs may not have
the same number of vertices.

\cite{pachauri2013solving} showed that the cycle consistency property is equivalent to
projecting the vertices into a discrete space. A matching between two vertices means that they
have the same image in this space. The dimension of this space is directly linked to the
diversity of the vertices in the set of graphs. This diversity reflects the attributes and/or
labels that are defined on the vertices. Recent methods like MatchALS~\citep{zhou2015multi}
implicitly use it by minimizing the rank the bulk permutation matrix.  Some others such as
HiPPI~\citep{bernard2019hippi} and GA-MGM~\citep{wang2020graduated} directly run the
optimization in the universe of vertex.  On the other side CAO~\citep{yan2015multi} and
Floyd~\citep{jiang2020unifying} optimize both affinity and the cycle consistency using a
graduated scheme. MLRWM-multi~\citep{park2016multi}, a multigraph extension of
RRWM~\citep{cho2010reweighted} is based on a multi-layer approach to ensure vertex
consistency.  Most of these methods relies on non-convex optimization as convex relaxation
is not trivial~\citep{fogel2013convex,swoboda2019convex} and may lead to inappropriate
solutions in some cases~\citep{aflalo2015convex,lyzinski2015graph}.  While there is now
many works using deep learning
methods~\citep{fey2020deepgraph,wang2020learning,rolinek2020deep,wang2021neural,liu2021stochastic}
for graph matching only a few tackle the multi-graph matching problem.

In this paper we focus on the unsupervised version of the multigraph matching problem
that is faced in the many applications where no ground truth labeling is available.
We propose an extended version of the framework proposed for
KerGM~\citep{zhang2019kergm} for multigraph matching with an appropriate projection to
fulfill the constraints. The induced optimization problem is solved using recent
non-convex optimization scheme which guarantees convergence under mild condition.  Results
on classic data set show promising results compared to other state-of-art unsupervised
methods.

Our main contributions are:
\begin{itemize}
\item An unsupervised multigraph matching method based on Lawler's QAP which allows kernel over both edge and
  vertex attributes while keeping a tractable computer memory load.
\item A flexible algorithm which allows a set of projection power to find \textit{cycle consistent} permutation matrices.
\item Experiments on real data-sets demonstrating that our approach is competitive with state-of-art results.
\end{itemize}

\section{Preliminaries about array operations and multi-graph matching}

Let $\{G_i\}_{i=1\ldots n}$ be a set of $n$ graphs with
$\forall i, G_i = \{V_i, E_i \subseteq V_i \times V_i, L^v, L^e\}$, i.e. a graph is
defined by a set of vertices, a set of edges and two functions:
$L^v : V \mapsto \mathbb{R}^{d_v}$ which gives the data vector associated to a vertex and
$L^e : E \mapsto \mathbb{R}^{d_e}$ which gives the data vector associated to an edge. We
also define $\boldsymbol{\mathrm{I}}_{m}$ as the identity matrix of size $m$ and
$\mathbb{1}_{m}$ the vector of $1$ of size $m$. For a matrix
$X \in \mathbb{R}^{n\times n}$, we denote for $i, j \in \{1,\ldots,n\}$ by $X[i,j]$ the
scalar at line $i$ and column $j$, and by $X[:,j]$ we denote the vector form by column $j$
(and reciprocally for lines).  First we introduce the notion of array operations in
Hilbert space, these tools will be useful to factorize~\eqref{eq:lawler} the edge affinity
matrices.

\subsection{Array operations in Hilbert space}\label{sec:array}

We propose a rapid and simplified introduction to $\mathcal{H}$-operations arrays in
Hilbert spaces presented in~\cite{zhang2019kergm}. Let
$\Phi, \Psi \in \mathbb{R}^{d\times m \times m}$ be 3d-arrays (or 3d-matrices). For
$\Phi \in \mathbb{R}^{d\times m \times m}$, we denote by $\Phi[l,i,j] \in \mathbb{R}$ the scalar at coordinates $l, i, j$ and by $\Phi[:,i,j] \in \mathbb{R}^d$
the vector at the coordinates $i, j$. $\forall X \in \mathbb{R}^{m \times m}$ we define
the following operations,
\begin{enumerate}
\item $\Phi^T \in \mathbb{R}^{d\times m \times m}$ with
  $\forall i, j,\ \Phi[:,i,j] = \Phi[:,j,i]$ (i.e. the transpose operation along the
  two last axis).
\item $\Phi * \Psi \in \mathbb{R}^{m \times m}$ where $\forall i, j$,
  \begin{align*}
    [\Phi * \Psi][i,j] = \sum_{k=1}^m \left\langle \Phi[:,i,k], \Psi[:,k,j] \right\rangle = \sum_{k=1}^m \sum_{l=1}^d \Phi[l,i,k] \Psi[l,k,j]~.
  \end{align*}
\item   $\forall X \in \mathbb{R}^{d \times m \times m}$,  $\Phi \odot X \in \mathbb{R}^{m \times m}$ where  $\forall i, j$,
  \begin{align*}
    [\Phi \odot X][i,j] = \sum_{k=1}^m \Phi[:,i,k] X[k,j] = \sum_{k=1}^m X[k,j] \Phi[:,i,k]~.
  \end{align*}
  Similarly, we define $X \odot \Phi \in \mathbb{R}^{m \times m}$ where,
  \begin{equation*}
  \forall i, j,\ [X \odot \Phi][i,j] = \sum_{k=1}^m  X[i,k] \Phi[:,k,j]~.
  \end{equation*}
\end{enumerate}
The last operator can be read as a parallelized matrix product over the first axis.

We recap several results from~\cite{zhang2019kergm} with the following proposition,
\begin{proposition}[\citet{zhang2019kergm}]\label{prop:form}
  Define the function
  $\left\langle .,. \right\rangle : \mathbb{R}^{d \times m \times m} \times \mathbb{R}^{d
    \times m \times m} \to \mathbb{R}$ such that
  $\forall \Phi, \Psi \in \mathbb{R}^{d\times m \times m}$,
  $\left\langle \Phi, \Psi \right\rangle = \mathrm{tr}(\Phi^T * \Psi)$. Then
  $\left\langle .,. \right\rangle$ is a inner product on
  $\mathbb{R}^{d \times m \times m}$ and the following equalities are true,
  \begin{itemize}
  \item Let $\Phi, \Psi \in \mathbb{R}^{d\times m \times m}$, we have
    $\forall X \in \mathbb{R}^{m \times m}$,
    $\left\langle \Phi \odot X, \Psi \right\rangle = \left\langle \Phi, \Psi \odot X^T
    \right\rangle$ and $\left\langle X \odot \Phi,  \Psi \right\rangle = \left\langle \Phi, X^T \odot \Psi
    \right\rangle$~.
  \item $\forall X, Y \in \mathbb{R}^{m \times m}$,
    $\Phi \in \mathbb{R}^{d\times m \times m}$, we have, $\Phi \odot X \odot Y = \Phi \odot (XY)$, and
    $Y \odot (X \odot \Phi) = (YX) \odot \Phi$~.
  \item Let $\Phi, \Psi \in \mathbb{R}^{d\times m \times m}$, we have
    $\forall X \in \mathbb{R}^{m \times m}$,
    $\left\langle \Phi, \Psi \odot X \right\rangle = \left\langle \Psi^T * \Phi, X
    \right\rangle$ and
    $\left\langle \Phi, X \odot \Psi \right\rangle = \left\langle \Phi * \Psi^T, X
    \right\rangle$~.
  \end{itemize}
\end{proposition}

\subsection{Building the optimization problem}

We assume that the graphs have the same number of vertices $m>0$, we add dummy vertices if
necessary. We define the constraints on cycle consistency in
Definition~\ref{def:cycle}.

\begin{definition}[Cycle-consistency~\citep{bernard2019hippi}]\label{def:cycle}
  Let $\mathcal{X} = \{X_{i,j} \in \mathcal{P}_{m, m}\}^n_{i,j=1}$ be the set of pairwise
  matchings in a collection of $n$ objects, where each $X_{i,j}$ is an element of the set
  of partial permutation matrices\footnote{All inequalities are element-wise.},
  \begin{equation}
    \label{eq:permset}
    \mathcal{P}_{m, q} = \{ X \in \{0, 1\}^{m \times q}\ :\ X\mathbb{1}_q \leq \mathbb{1}_m,   X^T\mathbb{1}_m \leq \mathbb{1}_q \}~.
  \end{equation}
  
  The set $\mathcal{X}$ is said to be cycle-consistent if for all
  $i, j, l \in \{1,\ldots, k\}$ it holds that: (i)
  $X_{i,i} = \boldsymbol{\mathrm{I}}_{m}$ (identity matching), (ii) $X_{i,j} = X^T_{j,i}$ (symmetry),
  and (iii) $X_{i,j} X_{j,l} \leq X_{i,l}$ (transitivity).
\end{definition}
The inequalities in this definition allows graphs of different sizes.

From Definition~\ref{def:cycle},~\cite{bernard2019hippi} formalized the links with the universe
of vertices with the following lemma,
\begin{lemma}[Cycle-consistency, universe points~\citep{bernard2019hippi}]\label{lem:cycle}
  The set $\mathcal{X}$ of pairwise matchings is cycle-consistent, if there exists a
  collection of rank $r$
  $\{X_l \in \mathcal{P}_{m, r}\ : \ X_l \mathbb{1}_r = \mathbb{1}_{m} \}^n_{l=1}$, such that
  for each $X_{i,j} \in \mathcal{X}$ it holds that $X_{i,j} = X_i X_j^T$.
\end{lemma}
Notice that the universe of vertices can be seen as a discrete embedding space for vertices.

From Lemma~\ref{lem:cycle} and using matrix factorization, we build our constraint set as follows,
\begin{equation}
  \label{eq:constraints}
  \mathcal{C}_r = \{ X \in \mathbb{R}^{nm\times nm}\ : \  X_{i,i} = \boldsymbol{\mathrm{I}}_m, X_{i,j} = X^T_{j,i} = X_i X_j^T \text{ with } X_i \in \mathcal{P}_{m,r} \}~.
\end{equation}
This formulation implies that $X\in\mathcal{C}_r$ is of rank $r$.

Let $K^{v}\in\mathbb{R}^{mn\times mn}$ be the full vertex affinity matrix built using a kernel on
vertices where each block $K^{v}_{i,j} \in \mathbb{R}^{m \times m}$ is the Gram matrix of a
vertex kernel applied on the data vectors (through $L^v$) on vertices of $G_i$ and $G_j$.  Let
$K^{e}\in\mathbb{R}^{mn^2\times mn^2}$ be the full edge affinity matrix built using a
kernel on edges. Using the Lawler's QAP, we aim to solve,
\begin{equation}
  \label{eq:fullqap}
  \max_{X\in\mathcal{C}_r} {\mathrm{vec}(X)}^{T} K^{e} \mathrm{vec}(X) + \mathrm{tr}(K^{v} X)~.
\end{equation}

\section{Proposed approach}

The limitation with~\eqref{eq:fullqap} is the size of $K^{e}$ which may be very large.  In
order to address this limitation, we use the framework proposed in Section~\ref{sec:array}
(following the ideas from~\cite{zhang2019kergm}) to rewrite $K^{e}$ by factorizing it using a
array $\Phi$ such that,
\begin{equation}
    \label{eq:rewrite}
    {\mathrm{vec}(X)}^{T} K^{e} \mathrm{vec}(X) = \left\langle \Phi \odot X, X \odot \Phi \right\rangle~.
\end{equation}

To build $\Phi$, we need to define it for each graph. For graph $G_k$,  let $\Phi_{k}^{d \times m \times m}$ be defined as follows,
\begin{equation}
  \label{eq:phi}
  \Phi_{k} [i, j] =
  \begin{cases}
    L^e(v_i, v_j) = \phi_{i, j} \in \mathbb{R}^d\quad &\text{if } (v_i, v_j) \in E_{k}~, \\
    0 \quad &\text{otherwise}~.
  \end{cases}
\end{equation}
$\phi_{i, j}$ is the data vector on the edge $(v_i, v_j)$. There many way to build these
vectors, for example Gaussian kernels can be approximated using Random Fourier
Features~\citep{rahimi2007random} (see also \citep{liu2021random} for other kernel
approximations). Then  $\Phi \in \mathbb{R}^{d \times nm \times nm}$ is given by the 3d-array
with the $\Phi_k$ on the diagonal,
\begin{equation}
  \label{eq:bigphi}
  \Phi = \left(\begin{array}{ccc} \Phi_1 & & 0  \\ & \Phi_2 & \\  0 & & \ddots \end{array}\right)
\end{equation}
This formulation has the advantage to be more compact than manipulating the full edge
affinity matrix, since the size of $\Phi$ is $d \times nm \times nm$ compare to
$(nm)^2 \times (nm)^2$ for $ K^{e}$.

Then we reformulate~\eqref{eq:fullqap} as
\begin{equation}
  \label{eq:newqap}
  \max_{X\in\mathcal{C}_r}  J(X) = \left\langle \Phi \odot X, X \odot \Phi \right\rangle + \mathrm{tr}(K^{v} X)~.
\end{equation}
This optimization problem is non-convex because of the two constraints (rank and
permutation), we need then a dedicated scheme to solve it.

\subsection{Optimization scheme}

Since~\eqref{eq:newqap} is a quadratic optimization problem with non-convex constraints,
we use a power optimization scheme. This scheme has the advantages to have almost no
hyper-parameters and show rather fast convergence rate and the non-convex constraints
are deal with a projector which gives an approximate solution. First we need the gradient of the objective function in~\eqref{eq:newqap}.

\begin{proposition}\label{prop:grad}
  The gradient of $J$ in~\eqref{eq:newqap} at $X$ is,
  \begin{equation}
    \label{eq:gradient}
    \nabla J(X) = K^{v} + (\Phi \odot X) * \Phi^T + \Phi^T * (X \odot \Phi)~.
  \end{equation}
\end{proposition}
\begin{proof}
  Let $J = J_1 + J_2$ with
  $J_1(X) = \left\langle \Phi \odot X, X \odot \Phi \right\rangle$ and
  $J_2(X) = \mathrm{tr}(K^{v} X)$. By definition of the gradient we have,
  \begin{equation*}
    \left\langle \nabla J_1(X), E \right\rangle = \lim_{t \to 0} \frac{1}{t} (J_1(X + tE) - J_1(X))~.
  \end{equation*}
  Using Proposition~\ref{prop:form} we have,
  \begin{align*}
    J_1(X+tE) &=  \left\langle \Phi \odot (X+tE), (X+tE) \odot \Phi \right\rangle~, \\
              &=  \left\langle \Phi \odot X+ t  \Phi \odot E, X \odot \Phi + t E \odot \Phi \right\rangle~, \\
              &=  \left\langle \Phi \odot X,  X \odot \Phi \right\rangle + t \left\langle  \Phi \odot E, X \odot \Phi + t E \odot \Phi \right\rangle~, \\
              &= \left\langle \Phi \odot X,  X \odot \Phi \right\rangle + t  \left\langle \Phi \odot X,  E \odot \Phi \right\rangle + t  \left\langle \Phi \odot E,  X \odot \Phi \right\rangle +\\
              &\qquad t^2 \left\langle \Phi \odot E,  E \odot \Phi \right\rangle~.
  \end{align*}
  From the gradient definition we yield,
  \begin{align*}
    \left\langle \nabla J_1(X), E \right\rangle &=  \left\langle \Phi \odot X,  E \odot \Phi \right\rangle + \left\langle \Phi \odot E,  X \odot \Phi \right\rangle~, \\
    &=  \left\langle  (\Phi \odot X) * \Phi^T + \Phi^T * (X \odot \Phi), E  \right\rangle~.
  \end{align*}
  By identification we have $\nabla J_1(X) = (\Phi \odot X) * \Phi^T + \Phi^T * (X \odot \Phi)$.
  As the gradient of $J_2(X) = \mathrm{tr}(K^{v} X)$ is $\nabla J_2(X) = K^{v}$, we produce~\eqref{eq:gradient}.
\end{proof}

With the gradient we can solve~\eqref{eq:newqap} using a power method scheme. The full method is
described in Algorithm~\ref{algo:power}. Since the gradient is a proxy for the quadratic
operator, this scheme is closely related to HiPPI~\citep{bernard2019hippi}. Assuming that
we are able to solve the projection step, the scheme converges to a stationary point.

\begin{proposition}
  Algorithm~\ref{algo:power} produces a monotone sequence and converges to a stationary point
  in a finite number of steps.
\end{proposition}
\begin{proof}
  This proof relies on the same arguments as for HiPPI~\citep{bernard2019hippi} for its
  Proposition~3. Let $\mathcal{U} = \mathrm{conv}(\mathcal{C}_r)$ (the convex hull of the
  constraint set). We can relax~\eqref{eq:newqap} in
  \begin{align}
   \notag
    \max_{X\in\mathcal{U}}  \left\langle \Phi \odot X, X \odot \Phi \right\rangle + \mathrm{tr}(K^{v} X) &=
    \max_{X\in\mathcal{U}}  \mathrm{vec}(X)^T K^e \mathrm{vec}(X) + \mathrm{tr}(K^{v} X), \\
   \label{eq:e1}  &\equiv \min_{X\in\mathbb{R}^{nm\times mn}} \imath_{\mathcal{U}}(X) - \underbrace{\mathrm{vec}(X)^T \Tilde{K} \mathrm{vec}(X)}_{h(X)}~,
  \end{align}
  with $\imath_{\mathcal{U}}$ the convex indicator function of set
  $\mathcal{U}$ (i.e. it return 0 for element from $\mathcal{U}$ and $+\infty$ otherwise) and $\Tilde{K}$ the matrix compound of $K^e$ with the addition of $\mathrm{vec}(K^v)$ on the diagonal.
  Since~\eqref{eq:e1} is a difference of two convex functions we can apply
  difference of convex (DC) programming updates rules~\citep{le2018dc} with an initial step $X_0$,
  \begin{align*}
    V_t &= \nabla_{X}h(X) = 2\ \mathrm{unvec}(\Tilde{K} \mathrm{vec}(X_{t}))~,\\
    X_{t+1} &= \argmin_{X\in\mathbb{R}^{nm\times nm}} \imath_{\mathcal{U}}(X) - h(X_t) - \langle X - X_t, V_t \rangle~,\\
    &= \argmin_{X\in\mathcal{U}} -\langle X, V_t \rangle~,
  \end{align*}
  where $\mathrm{unvec}$ is the inverse of $\mathrm{vec}$.
  $\mathcal{U}$ is linked with the Birkhoff polytope with the permutation matrix as
  extreme points~\citep{maciel2003global,birdal2019probabilistic}. Since the maximum of the
  linear objective over a compact set is attained at its extreme points, we get,
  \begin{align*}
    X_{t+1}  &= \argmin_{X\in\mathcal{U}} -\langle X, V_t \rangle~, \\
    & =  \argmin_{X\in\mathcal{U}} \| V_t - X \|^2_F = \mathrm{proj}_{\mathcal{C}_r} (V_t)~,
  \end{align*}
  where $\langle X, X \rangle = nm$ since $X$ is formed from permutation matrices. Then we get from DC programming property that the sequence $\{X_t\}^T_{t=0}$ is a increasing
  sequence for $J$ in~\eqref{eq:newqap}. Since $\mathcal{C}_r$ is a finite set, $J$ is
  bounded above for any $X\in\mathcal{C}_r$. Furthermore as
  $\forall t > 0, X_t \in \mathcal{C}_r$, the increasing sequence $\{X_t\}^T_{t=0}$
  converges to a stationary point.
\end{proof}

\begin{algorithm2e}[htb]
  \SetAlgoLined
  \SetKwInOut{Input}{Input}\SetKwInOut{Output}{Output}
  \SetKwInOut{Initialize}{Initialize}

  \Input{Affinity vertices matrix $K^{v}$, edges' data 3d-array $\Phi$, tolerance $\tau$, maximal number of iterations $T$, rank $r$.}
  \Output{$X_{T}$ the bulk permutation matrix.}
  $X_0 \leftarrow 0$\;
  \For{$t \leftarrow 1$ \KwTo $T$}{
    $\hat{X}_t \leftarrow \nabla J(X_t)$\tcp*{Using Proposition~\ref{prop:grad}}
    $X_{t+1} \leftarrow \text{proj}_{\mathcal{C}_r} (\hat{X}_t)$\;
    \If{$\| X_{t+1} - X_t \| < \tau$}{
      break\;
    }
    $t \leftarrow t+1$\;
  }
\caption{Power method for solving~\eqref{eq:newqap}}\label{algo:power}
\end{algorithm2e}

\textbf{Complexity of Algorithm~\ref{algo:power}}: The computational complexity of the
algorithm mainly depends on both the computation of the gradient and the projection onto
the set of permutations. Since we are dealing with 3d-arrays of size $d\times m \times n$
the complexity of both $*$ and $\odot$ operations is $\mathcal{O}(d(nm)^2)$. The
complexity of the permutation depends on the selected approach to solve the
optimization. We discuss this point in the next section.

\textbf{Initialization}: As many non-convex methods our method is sensitive to the
initialization. For example methods like HiPPI~\citep{bernard2019hippi} or
GA-MGM~\citep{wang2020graduated} cannot start with an uniformly valued matrix. Since they
are working in the universe of vertices, the initialization impacts the projection step
and a method like the Hungarian method will just put the vertex in order (i.e. the vertex 1
will match other vertex 1 and so on). The algorithm ends-up being stuck in a solution where
the permutation matrices are simply the identity. To avoid such an issue, GA-MGM for
example, uses a random initialization, leading to a non deterministic optimization
method. For our method such initialization is less a problem. If we initialize with the
null matrix, the algorithm will only use the vertex affinity matrix to recover the
permutations. Such initialization is relevant if the affinity is correctly tuned and it
makes our method deterministic.

\textbf{Comparison with closely related methods}: While our method is not the first
multi-graph method to include edge attributes, it is the first to propose a tractable way
of dealing with real-world data. For example, methods like Floyd~\citep{jiang2020unifying}
or deep learning methods like NGM~\citep{wang2021neural} directly work with Lawler's
QAP~\eqref{eq:lawler}. This leads to manipulate matrices whose size are quadratic in the
number of vertices, so they only consider pairwise matching for most set of graphs. In
comparison, with the arrays, we are able to manipulate matrices whose sizes linearly
depends on the number of vertices. The cycle consistency constraints can then consider the full set
of graphs.

\subsection{Projection over the set of constraints}

Optimizing on $\mathcal{C}_r$ is a NP-Hard problem. Furthermore as mentioned in several
publications~\citep{bernard2019hippi,shi2020robust}, most of the approximation methods are
very sensitive to the initialization and some require to provide a reference graph
(mSync~\citep{pachauri2013solving} for example). Recent methods include:
\begin{description}
\item[mSync]~\citep{pachauri2013solving} uses an eigen decomposition and a reference graph
  (the first graph in the algorithm presented in the paper) combined with the Hungarian
  method to recover the permutations.
\item[MatchEIG]~\citep{maset2017practical} is an alternative to mSync with no reference
  graph but still combines eigen decomposition with the Hungarian method.  As noticed
  in~\cite{bernard2021sparse}, the cycle consistency is not guaranteed.
\item[Birkhoff-RLMC]~\citep{birdal2019probabilistic} uses a probabilistic method to
  estimate the permutation matrix from a noisy observation. This method relies on an
  optimization onto Riemannian manifold.
\item[IRGCL]~\citep{shi2020robust} combines several methods to estimate the permutations
  from a noisy permutation matrix. This method requires a reference graph for cycle
  consistency. They also estimate the projection of each vertex in the universe of
  vertices.
\item[SQAP]~\citep{bernard2021sparse} is based on QR decomposition and a power method to
  recover the permutations from a noisy permutation matrix. As in IRGCL, they estimate the
  projection of each vertex in the universe of vertices and guarantee the cycle
  consistency.
\item[Generalized power method]~\citep{ling2020improved,ling2022near} can be seen as an
  iterative extension of MatchEIG where the result is refined through a power method
  scheme. As for MatchEIG the cycle consistency is not guaranteed.
\end{description}

In this paper, we propose to use three methods: MatchEIG in the version presented in
Algorithm~\ref{algo:matcheig}, the iterative version of MatchEIG
(Algorithm~\ref{algo:gpow}) which is a generalized power method and we expect it to
produce better results~\citep{ling2020improved} and IRGCL~\citep{shi2020robust}. To avoid
the need of a reference graph, we propose to apply IRGCL combined with SQAP for solving
the inner problem of permutation estimation. We apply it on the permutation matrix
estimated by MatchEIG. While IRGCL is more accurate than MatchEIG, its complexity may
limit its scalability to large data-sets. Notice that all these methods require to have a estimation
of the rank $r$, some are more sensitive to a bad value than others~\citep{maset2017practical} and
finding easily a good guess or the optimal value remains an open question.

\begin{algorithm2e}[htb]
  \SetAlgoLined
  \SetKwData{Left}{left}\SetKwData{This}{this}\SetKwData{Up}{up}
  \SetKwInOut{Input}{Input}\SetKwInOut{Output}{Output}
   \SetKwInOut{Initialize}{Initialize}

  \Input{Perturbed bulk permutation matrix $X$, rank $r$.}
  \Output{$\hat{X}$ an estimation of the bulk permutation matrix.}
  $U \Sigma U^T \leftarrow \mathrm{EIG}(X)$\;
  $U \leftarrow U \sqrt{\Sigma}$\;
  \For{$i \leftarrow 1$ \KwTo $n$}{
   \For{$j \leftarrow 1$ \KwTo $n$}{
        $Z \leftarrow U[im:(i+1)m, 1:r] U^T[jm:(j+1)m, 1:r]$\;
        $\hat{X}_{i,j} \leftarrow \mathrm{Hungarian(Z)}$\;
    }
  }
  \caption{MatchEIG~\citep{maset2017practical} for computing an approximation of the
    permutation matrix where EIG$(X)$ is function which computes the eigen decomposition
    of $X$ and Hungarian$(Z)$ is the Hungarian method for estimating the permutations from
    $Z$.}\label{algo:matcheig}
\end{algorithm2e}

\begin{algorithm2e}[htb]
  \SetAlgoLined
  \SetKwData{Left}{left}\SetKwData{This}{this}\SetKwData{Up}{up}
  \SetKwInOut{Input}{Input}\SetKwInOut{Output}{Output}
   \SetKwInOut{Initialize}{Initialize}

  \Input{Perturbed bulk permutation matrix $X$, rank $r$, tolerance $\tau$, maximal number of iterations $T$.}
  \Output{$\hat{X}$ an estimation of the bulk permutation matrix.}
  $Z_0 \leftarrow \mathrm{MatchEIG}(X, r)$\;
  \For{$i \leftarrow 1$ \KwTo $T$}{
    $Z_t \leftarrow \mathrm{MatchEIG}(X Z_{t-1}, r)$\;
    
        \If{$\| Z_{t-1} - Z_t \| < \tau$}{
            break\;
        }
  }
  $\hat{X} \leftarrow Z_T$\;
  \caption{GPow: an adaptation of the generalized power
    method~\citep{ling2020improved,ling2022near} for the bulk permutation matrix
    estimation.}\label{algo:gpow}
\end{algorithm2e}

\textbf{Complexity:} We report here only the complexity of three methods that will be used
in Section~\ref{sec:exp}. As the version of MatchEIG presented by
Algorithm~\ref{algo:matcheig} is based on an eigen decomposition and the Hungarian method
the complexity is at least $\mathcal{O}(n^2m^3)$. SQAP is mostly based on QR decomposition
and GPow iterate MatchEIG-like methods. These methods have a complexity of
$\mathcal{O}(Tn^2m^3)$ with $T$ the maximal number of iterations. IRGCL is at least as
complex because it relies on methods like mSync at each iteration. Thus, while MatchEIG
may be the fastest method, the others may lead to better results that fully respect the
cycle consistency.

\section{Experiments}\label{sec:exp}

We extensively benchmark our method on four data-sets. In addition to a synthetic set of
graphs, we consider two multi-image data-sets that are commonly used for graph matching
comparison: Willow, PascalVOC\footnote{Some
  statistics on these sets can be found in~\cite{fey2020deepgraph}.}. We compare our
method against several state-of-art unsupervised methods: HiPPI~\citep{bernard2019hippi},
MatchEIG~\citep{maset2017practical} (without the thresholding step),
GA-MGM~\citep{wang2020graduated} and Floyd~\citep{jiang2020unifying}. We also compare the
performances of our method against three state-of-art supervised methods for pairwise
graph matching: DGMC~\citep{fey2020deepgraph}, NGM~\citep{wang2021neural} and
SIGMA~\citep{liu2021stochastic}. Both methods are based on deep learning to learn the best
representation for matching.

\textbf{Data-set management}: All the computer vision data-sets are provided by
\textit{pytorch-geometric}\footnote{\url{https://pytorch-geometric.readthedocs.io/en/latest/index.html}}. We
directly use the keypoint and their attributes (descriptor and position). For the
construction of the graphs from the keypoints, we follow the procedure
from~\cite{fey2020deepgraph} where the graphs are built using a Delaunay tessellation with
an isotropic distance. For Willow and PascalVOC, the descriptors are taken for each
keypoint from the concatenated output of \textit{relu4\_2} and \textit{relu5\_1} on a
VGG16 network trained on ImageNet~\citep{SimonyanZ14a}.

\textbf{Kernel setting}: For the synthetic graphs experiments we only use linear kernels
to avoid hyper-parameter setting. For Willow and PascalVOC, we use the Gaussian
kernel for computing the vertices weight as in~\cite{bernard2019hippi}. We follow also the
same protocol as in~\cite{bernard2019hippi} to compute the weights on the edges. The
weights are computed using a Gaussian kernel applied to the distance between the vertices
and the variance is given by the median of the minimal distances. As kernel on the edges,
we use the Random Fourier Features (RFF)~\citep{rahimi2007random} as proposed for
KerGM~\citep{zhang2019kergm}. For all experiments we take a dimension of 100 for the random
features.

\textbf{Initialization}: For all our experiments, we initialize our method using the null
matrix, which is equivalent to an initialization using the projection method on the vertex
affinity matrix.

\textbf{Dummy vertices}: For all experiments we add dummy vertices when the size of the
graphs varies. In practice, we add unconnected vertices with attribute values far from the
legitimates ones (the same for all dummy vertices). These vertices are removed before
computing the score, the vertices that are matched to a dummy vertex are then considered
as unmatched. We thus take such configuration into account when defining the score.

\textbf{Convergence setting}: For Algorithm~\ref{algo:power} we set the maximal number of
iterations to 100 and the tolerance to $10^{-2}$. The algorithm stops when one of the
conditions is met. For IRGCL we use the same setting and parameters as proposed
in~\cite{shi2020robust}. For Algorithm~\ref{algo:gpow} we set the maximal number of
iterations at 100 with a tolerance of $10^{-3}$.

\textbf{Similarity scores:} We compute the precision and recall directly from the ground
truth bulk permutation matrix $X_{truth}$. Let $X$ be the current estimated bulk
permutation matrix, the score are then computed as follows:
\begin{align}
    \notag\mathrm{precision} &= \frac{\langle X_{truth} - \boldsymbol{\mathrm{I}}, X - \boldsymbol{\mathrm{I}}\rangle}{\| X - \boldsymbol{\mathrm{I}} \|^2_F},\qquad 
    \mathrm{recall} = \frac{\langle X_{truth} - \boldsymbol{\mathrm{I}}, X  - \boldsymbol{\mathrm{I}}\rangle}{\| X_{truth}  - \boldsymbol{\mathrm{I}}\|^2_F} \\
    \label{eq:score} \mathrm{F1-score} &= \frac{2 * \mathrm{precision} * \mathrm{recall}}{\mathrm{precision} 
    + \mathrm{recall}}
\end{align}

\subsection{Synthetic data-set}\label{sec:robustness}

We use a synthetic data-set to assess the robustness of our method to different
perturbations. In this purpose we adapt the protocol
from~\cite{gold1996graduated,zhang2019kergm} where the graphs are generated from an
Erd\H{o}s-R\'enyi model. For this experiment, we generate graphs with 50 vertices and a
probability of edges of $0.05$ (also named \textit{connectivity density}
in~\cite{gold1996graduated}). For each test, we report the mean value across 20
graphs. For each graph, the attributes of the vertices and the edges are random vectors built
from the uniform distribution $\mathcal{U}(0, 1)$. The dimension of these vectors is set
to $10$.

First, we assess the sensitivity of our method with respect to noise on the
attributes. For each graph in the set we compute nine shuffled versions to make a total of
10 graphs. We add the same additive Gaussian noise on the attributes of the vertices and
edges, with different variances. We set the rank to 50, the number of vertices. We compare
our method against MatchEIG~\citep{maset2017practical} where only the vertices are taken
into account. Figure~\ref{fig:noise}(left) shows the evolution of the F1-score as a
function of the variance of the noise. We show only our results with MatchEIG as
projector, we also test with IRGCL and GPow but they yield similar results, so for
visualization purpose we selected only one method. Our method is more robust to noise
since it uses both edges and vertices. It is interesting to see that it is robust to small
and mild noise values while the performance severely degrades with higher noise
levels. Compared to MatchEIG, the robustness likely comes from taking the edges into
account, as long as the attributes on edges are reliable.

\begin{figure}[tb]
    \centering
    \includegraphics[width=0.44\linewidth]{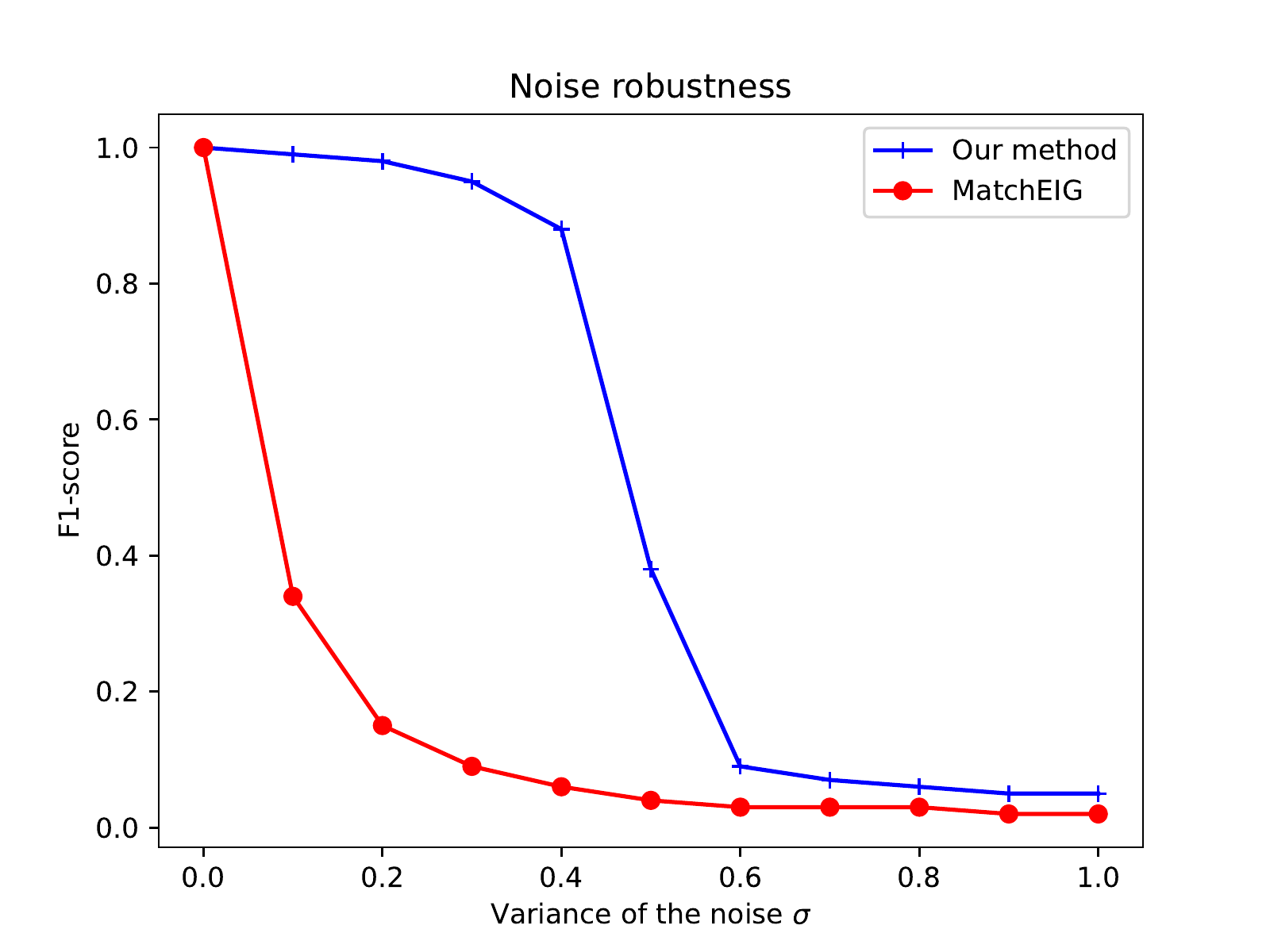}
    \includegraphics[width=0.44\linewidth]{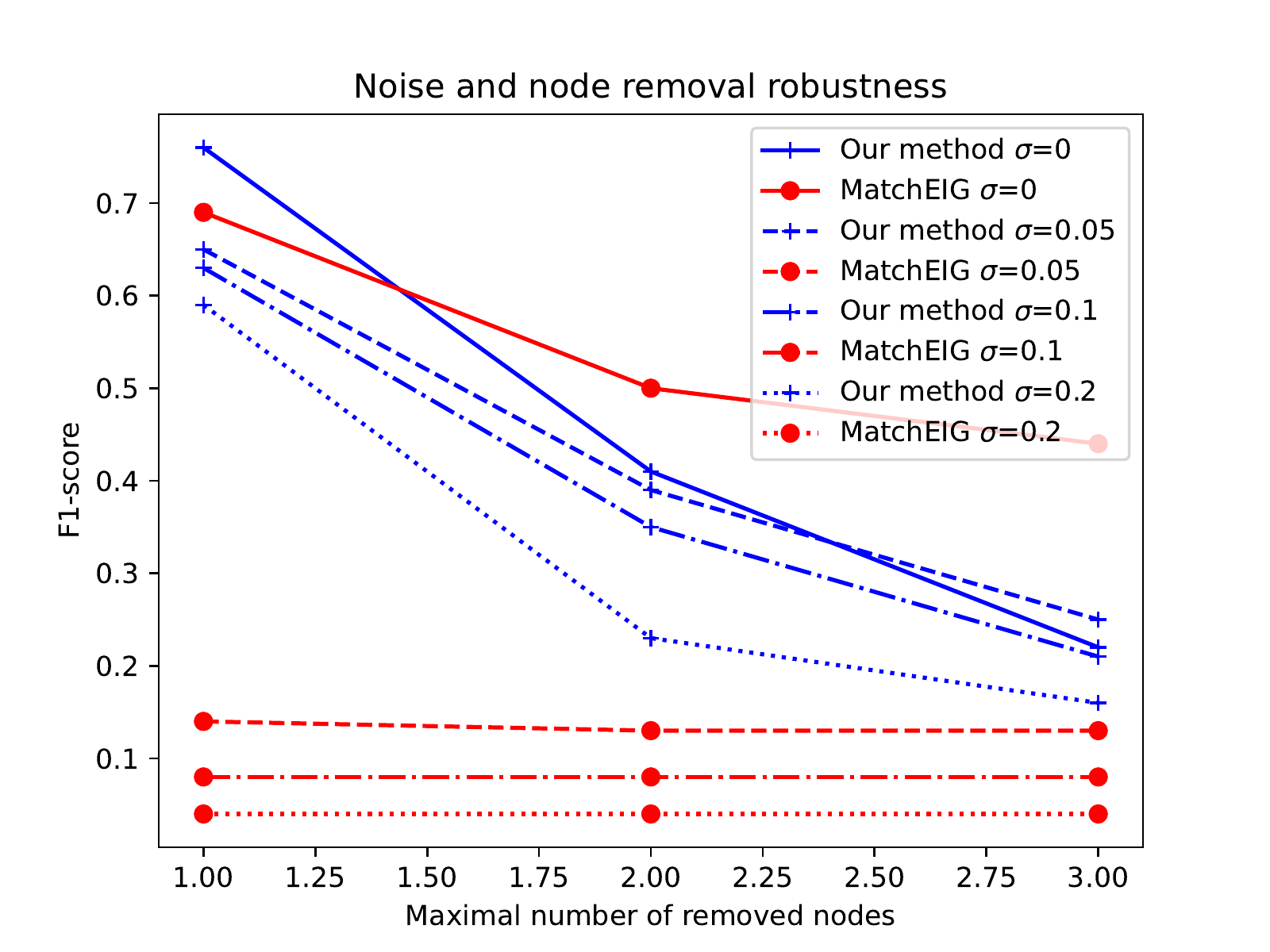}
    \caption{(left) The robustness of our method compare to MatchEIG following the
      variance of the noise. (right) The robustness of our method compare to MatchEIG
      following the maximal number of removed vertices and the variance of the noise.}
    \label{fig:noise}
\end{figure}

Secondly, we assess the robustness of our method against vertex removal. For this, we
remove vertices at random until a maximal number is reached. This means that for a given
set of graphs, the number of vertices removed is not the same for all the graphs but
cannot exceed the specified maximum. The removed vertices are disconnected from the other
vertices and transformed into dummy in order to preserve an equal number of vertices. We
also add different levels of noise on the attributes (same variance for edge and vertex
data vectors). The results are presented by Figure~\ref{fig:noise}(right). In absence of
noise, MatchEIG shows a better robustness to vertex removal. But when we add the noise,
our method performs better. Note that the matching quality degrades severely with
increased maximal number of removed vertices. This is expected since removing vertices
leads to changes in the edges and but also in the topology of the graphs.

\subsection{Willow data-set}\label{sec:willow}

Willow is an image data-set composed of 5 classes of objects. For each class we have a
minimum of 40 graphs with exactly 10 vertices. Here we use the same edge
attributes as in~\cite{bernard2019hippi} where the weights are computed using the
median distance between vertices.

In order to find an appropriate set of hyper-parameters we run a grid-search on the
variance for the Gaussian kernel on the vertices and the parameter of the RFF.
The tested values range from $10$ to $100$ for an additive step of $10$ for the vertices and from $10^{-6}$
to $1$ with a multiplicative step of $0.1$ for the edges.
For each set of values, we take at random 10 graphs from the \textit{car} class and
apply our method to estimate the F1-score. We repeat the process 10 times and report the mean of the scores. We use the set with the best mean score ($80$ for vertices and $10^{-5}$ for edges).

\begin{table}[tp]
  \centering
  {\small
  \begin{tabular}{l||c|c|c|c|c}
    Method/Class & Car & Duck & Face & Motorbike & Winebottle \\
    \hline \hline 
    MatchEIG~\citep{maset2017practical} & 23.4 & 36.1 & 47.6 & 19.4 & 21.1 \\
    HiPPI~\citep{bernard2019hippi} & 74.0 & 88.0 & 100 & 84.0 & 95.0 \\
    GA-MGM~\citep{wang2020graduated} & 74.6 & 90.0 & 99.7 & 89.2 & 93.7 \\
    Floyd-c~\citep{jiang2020unifying} & 85.0 & 79.3 & 100 & 84.3 & 93.1 \\
    DGMC~\citep{fey2020deepgraph} & 95.53 & 93.0 & 100 & 99.4 & 99.39 \\
    NGM-v2~\citep{wang2021neural} &97.4 & 93.4 & 100 & 98.6 & 98.3 \\
    \hline Our (MatchEIG) & 89.6 & 80.1 & 100 & 82.0 & 93.0 \\
    Our (GPow) & 90.3 & 78.6 & 100 & 80.7 & 96.1 \\
    Our (IRGCL) & 83.5 & 74.5 & 100 & 83.7 & 93.8 \\
  \end{tabular}
  }
  \caption{F1-score of the different approaches on the Willow data-set. DGMC and NGM-v2 are
    supervised method where the vertices attributes are learned using deep learning
    framework. The others methods are  unsupervised. We remind that the IRGCL method is
    combined with SQAD and initialized with MatchEIG result.}
  \label{tab:willow}
\end{table}

We present the results in Table~\ref{tab:willow}. Here we use three methods for the
projection onto permutations set: MatchEIG, GPow and IRGCL (with SQAD for the inner
optimization step). For all methods expect MatchEIG we report the score from the
articles. For MatchEIG we used our own implementation. The results show that we are
competitive compared to other state-of-art methods. On the projection side, IRGCL leads to
better results on \textit{motorbike} while GPow is better on \textit{winebottle}. Between
the three methods, MatchEIG gives the best average results and is a little better than
unsupervised state-of-art methods like Floyd-c~\citep{jiang2020unifying}.

\subsection{PascalVOC data-set}

PascalVOC is an image data-set composed of 20 classes of objects with key points on each image. This database is
challenging since the number of vertices vary greatly inside a class and there are also
outliers. The number of vertices for a graph can vary from 1 to 16. For this data-set with follow
the protocol from~\cite{fey2020deepgraph} (it differs from the protocol
of~\cite{wang2021neural} as they have different limits for the number of vertices). For
comparison purpose we only use the \textit{test} sets of each category.

\begin{table}[tb]
    \centering
    {\tiny
      \begin{tabular}{l||c|c|c|c|c|c|c|c|c|c}
        Method/Class & Aero&Bike&Bird&Boat&Bottle&Bus&Car&Cat&Chair&Cow \\
        \hline \hline MatchEIG~\citep{maset2017practical} & 9.0 & 15.3 & 15.8 & 16.8 & 19.8 & 23.5 & 13.3 & 13.0 & 11.5 & 9.4 \\
        NGM (unsup)~\citep{wang2021neural} & 30.8  & 42.5 & 44.3 & 33.8 & 39.8 &52.2 & 49.2 & 53.9 & 27.5 & 42.4 \\
        SIGMA~\citep{liu2021stochastic} & 55.1 & 70.6 & 57.8 & 71.3 & 88.0 & 88.6 & 88.2 & 75.5 & 46.8 & 70.9  \\
        DGMC~\citep{fey2020deepgraph} & 50.1 & 65.4 & 55.7 & 65.3 & 80.0 & 83.5 & 78.3 & 69.7 & 34.7 & 60.7  \\
        NGM-v2~\citep{wang2021neural} & 61.8 & 71.2 & 77.6 & 78.8 & 87.3 & 93.6 & 87.7 & 79.8 & 55.4 & 77.8  \\
        \hline Ours (Eig) & 10.7 & 41.6 & 24.6 & 24.3 & 47.6 & 31.0 & 19.0 & 24.1 & 14.4 & 10.4 \\
        \hline \hline
        Method/Class & Table & Dog & Horse&M-Bike&Person&Plant&Sheep&Sofa&Train& TV \\
        \hline \hline MatchEIG~\citep{maset2017practical} & 19.2 & 12.2 & 9.6 & 11.7 & 6.5 & 20.1 & 10.6 & 15.3 & 28.0 & 36.5  \\
        NGM (unsup)~\citep{wang2021neural} & 29.3 & 49.1 & 45.1  & 45.1  & 24.0 & 48.3 & 49.9 & 29.9  & 70.2  & 73.3  \\
        SIGMA~\citep{liu2021stochastic} & 90.4 & 66.5 & 78.0 & 67.5 & 65.0 & 96.7 & 68.5 & 97.9 & 94.3 & 86.1 \\
        DGMC~\citep{fey2020deepgraph} & 70.4& 59.9 & 70.0 & 62.2 & 56.1 & 80.2 & 70.3 & 88.8 & 81.1 & 84.3 \\
        NGM-v2~\citep{wang2021neural} & 89.5& 78.8 & 80.1 & 79.2 & 62.6 & 97.7 & 77.7 & 75.7 & 96.7 & 93.2 \\
        \hline Ours (Eig)  & 15.4& 15.1 & 15.4 & 20.3 & 8.0 & 60.2 & 10.6 & 17.2 & 48.3 & 62.3 \\ 
      \end{tabular}}
    \caption{F1-score of the different approaches on the PascalVOC data-set.}
    \label{tab:pascalvoc}
\end{table}

For the hyper-parameters estimations we use the same protocol as for Willow (Section~\ref{sec:willow})
but using the \textit{horse} class. The optimal parameter were $60$ for vertices and $0.01$ for edges. We
report the results in Table~\ref{tab:pascalvoc}. Except MatchEIG where we use our own implementation,
the results of the other methods are taken from the articles. Notice that we also report the
results from NGM-v2 using an unsupervised setting with the same VGG16 data of the vertices.
While supervised methods  clearly perform better on this set, we are competitive with the two unsupervised 
methods. For computational reasons we did not apply IRGCL or GPow for the projection step. Interestingly while our method 
 generally outperform MatchEIG, there are some classes where it is the reverse (e.g. \textit{tables}).
This could be explained by the way the edges are built (Delaunay tessellation). For some classes we are even better than the
unsupervised version of NGM-v2 (\textit{bottle} and \textit{plant}).

These results are in line with our experiment on the robustness in
Section~\ref{sec:robustness}. Since PascalVOC graphs are affected by noise and vertex
suppression, our method performs badly on the most degraded category. Furthermore some
categories are only composed by few graphs, for the \textit{test} set. The cycle
consistency is then insufficient to recover the good matches.

\section{Conclusion and future work}

We propose a generalization of the KerGM approach~\citep{zhang2019kergm} for multi-graph
matching. Our approach allows for more flexibility than previous methods for dealing
attributes on edges. While others methods directly deals with Lawler's QAP, our approach
relies on more efficient matrices and arrays and can be applied to larger and more
complex data-sets. Contrary to other methods~\citep{bernard2019hippi,wang2020graduated}, we
do not optimize in the universe of vertices, but directly manage the projection onto the
permutation sets. This avoids some common limitations induced by the definition of the
universe of vertices. In addition, our approach benefits from recently published methods to
estimate the projection in reasonable computational times.

The kernel framework is very flexible and allows for specific definitions of the affinity
on vertices and edges. For the edges we only use the classical Random Fourier Feature, but
one can use more adapted features~\citep{li2019towards} while keeping a tractable memory
load. As an unsupervised method, the matching relies essentially on the constraints and
the experiments show it is sufficient to deal with mild level of noise and almost
homogeneous, in term of number of vertices, sets of graphs. Managing more complex sets
requires either supervised learning or additional constraints.

Potential improvements include building a stochastic version of the method in order to be
able to handle very large sets of graphs, for example by adapting stochastic DC
methods~\citep{xu2019stochastic}. We also seek to improve the projection step while keeping
an acceptable computational burden.


\acks{This work received support from the French government under the France 2030 investment plan, as part of the Initiative d'Excellence d'Aix-Marseille Université - A*MIDEX.}

\bibliography{multigraph}


\end{document}